\documentclass[letterpaper, 10 pt, conference]{ieeeconf}  

\IEEEoverridecommandlockouts                              

\usepackage{amsmath}
\usepackage{esint}
\usepackage{amssymb}
\usepackage{amsthm}                                                                   \usepackage{booktabs}   
\usepackage{multirow}   
\usepackage{graphicx}   
\usepackage{mathtools}
\usepackage{derivative}
\usepackage{xcolor}
\usepackage{bm}
\usepackage{physics} 
\usepackage[ruled,vlined]{algorithm2e}
\usepackage{optidef}
\usepackage{graphicx}   
\usepackage{caption}    
\usepackage{subcaption} 

\usepackage{siunitx} 
\usepackage[noadjust]{cite}
\usepackage{url}
\usepackage[hidelinks]{hyperref}

\usepackage{graphicx}
\usepackage[export]{adjustbox} 
\graphicspath{{figures/}}
\usepackage{stfloats}

\newtheorem{proposition}{Proposition}

\theoremstyle{definition}
\newtheorem{definition}{Definition}

\newtheorem{remark}{Remark}

\usepackage{bm}

\DeclareMathOperator*{\argmin}{arg\,min}

\title{\LARGE \bf
VLPSA: Vision-Language-Poisson-Safe Actions \\ for Full-Body Safety of Learned Policies}

\author{Meg Wilkinson$^{*}$, Emily Fourney$^{*}$, Joel W. Burdick, Aaron D. Ames%
\thanks{$^{*}$These authors contributed equally to this work.}%
\thanks{All authors are with the Department of Computing and Mathematical Sciences, California Institute of Technology, Pasadena, CA 91125, USA. {\tt\small \{mwilkins, efourney, burdick, ames\}@caltech.edu}}}

\begin{document}

\maketitle
\thispagestyle{empty}
\pagestyle{empty}




\newcommand{\re}{\mathbb{R}}


\newcommand{\dt}{\mathrm{d}t}
\newcommand{\dy}{\mathrm{d}y}
\newcommand{\dx}{\mathrm{d}x}
\newcommand{\dtau}{\mathrm{d}\tau}
\newcommand{\Cc}{\mathcal{C}}
\newcommand{\Cce}{\mathcal{C}_\varepsilon}
\newcommand{\he}{h_{\varepsilon}}

\newcommand{\Ac}{\mathcal{A}}
\newcommand{\pCc}{\partial \mathcal{C}}
\newcommand{\Bc}{\mathcal{B}}
\newcommand{\Tc}{\mathcal{T}}
\newcommand{\Dc}{\mathcal{D}}
\newcommand{\Oc}{\Omega}
\newcommand{\Occ}{\overline{\Omega}}
\newcommand{\pOc}{\partial \Omega}
\newcommand{\Ocext}{\Oc_\mathrm{ext}}
\newcommand{\Ocint}{\Oc_\mathrm{int}}
\newcommand{\Hc}{\mathcal{H}}
\newcommand{\Kc}{\mathcal{K}}
\newcommand{\Fc}{\mathcal{F}}
\newcommand{\Mc}{\mathcal{M}}
\newcommand{\Nc}{\mathcal{N}}
\newcommand{\Pc}{\mathcal{P}}
\newcommand{\Uc}{\mathcal{U}}
\newcommand{\Sc}{\mathcal{S}}
\newcommand{\Xc}{\mathcal{X}}
\newcommand{\Yc}{\mathcal{Y}}
\newcommand{\Vc}{\mathcal{V}}
\newcommand{\Zc}{\mathcal{Z}}
\newcommand{\Lc}{\mathcal{L}}
\newcommand{\Rm}{\mathcal{\mathbb{R}}}
\newcommand{\R}{\mathcal{\mathbb{R}}}
\newcommand{\Sp}{\mathcal{\mathbb{S}}}

\newcommand{\divv}{\nabla \cdot \vec{\bv}}
\newcommand{\hs}{h_\mathrm{\Sc}}

\newcommand{\vr}{\varepsilon}

\newcommand{\des}{{\operatorname{des}}}
\newcommand{\on}{{\operatorname{on}}}
\newcommand{\off}{{\operatorname{off}}}
\newcommand{\fl}{{\operatorname{FL}}}
\newcommand{\Lie}{\mathcal{L}}
\newcommand{\qp}{{\operatorname{QP}}}

\newcommand{\ie}{i.e., }
\newcommand{\todo}[1]{{\color{cyan} Todo: #1}}

\newcommand{\ba}{\mathbf{a}}
\newcommand{\bb}{\mathbf{b}}
\newcommand{\be}{\mathbf{e}}
\renewcommand{\bf}{\mathbf{f}} 
\newcommand{\bff}{\mathbf{f}}
\newcommand{\bg}{\mathbf{g}}
\newcommand{\bk}{\mathbf{k}}
\newcommand{\bp}{\mathbf{p}}
\newcommand{\bq}{\mathbf{q}}
\newcommand{\bu}{\mathbf{u}}
\newcommand{\bv}{\mathbf{v}}
\newcommand{\bvv}{\vec{\mathbf{v}}}
\newcommand{\bn}{\mathbf{n}}
\newcommand{\hbn}{\hat{\mathbf{n}}}

\newcommand{\bx}{\mathbf{x}}
\newcommand{\bz}{\mathbf{z}}
\newcommand{\br}{\mathbf{r}}
\newcommand{\bA}{\mathbf{A}}
\newcommand{\bB}{\mathbf{B}}
\newcommand{\bD}{\mathbf{D}}
\newcommand{\bC}{\mathbf{C}}
\newcommand{\bF}{\mathbf{F}}
\newcommand{\bJ}{\mathbf{J}}
\newcommand{\bG}{\mathbf{G}}
\newcommand{\bK}{\mathbf{K}}
\newcommand{\bP}{\mathbf{P}}
\newcommand{\bW}{\mathbf{W}}
\newcommand{\bw}{\mathbf{w}}
\newcommand{\bd}{\mathbf{d}}
\newcommand{\bvy}{\vec{\by}}
\newcommand{\bty}{\tilde{\by}}
\newcommand{\bbeta}{\boldsymbol{\eta}}
\newcommand{\mb}[1]{\mathbf{#1}}

\newcommand{\bY}{\mathbf{Y}}
\newcommand{\by}{\mathbf{y}}
\newcommand{\byobs}{\mathbf{y}_\mathrm{obs}}

\newcommand{\bxd}{\bx_\mathrm{d}}
\newcommand{\bxobs}{\bx_\mathrm{obs}}
\newcommand{\md}{\mathrm{d}}

\newcommand{\Uxd}{U_{\mathrm{d}}}
\newcommand{\Uobs}{U_{\mathrm{obs}}}
\newcommand{\Uapf}{U_{\mathrm{APF}}}

\newcommand{\GradUxd}{\nabla U_{\mathrm{d}}}
\newcommand{\GradUobs}{\nabla U_{\mathrm{obs}}}
\newcommand{\GradUapf}{\nabla U_{\mathrm{APF}}}

\newcommand{\cmax}{c_\mathrm{max}}
\newcommand{\cmin}{c_\mathrm{min}}

\newcommand{\hn}{h_\mathrm{n}}
\newcommand{\Dhn}{D h_\mathrm{n}}
\newcommand{\Dh}{D h}
\newcommand{\Dhd}{D h_\mathrm{d}}

\begin{abstract}

Vision-language-action (VLA) models enable increasingly general-purpose robotic manipulation, but such learned policies do not provide safety guarantees for collision avoidance---especially in environments outside of training distributions. This work presents Vision-Language-Poisson-Safe Actions (VLPSA), a safety filtering framework that provides full-body safety for VLA policies in cluttered and dynamic environments without retraining. VLPSA synthesizes Poisson Safety Functions (PSF) online from perception data, yielding a Control Barrier Function (CBF) that is enforced through a CBF-QP safety filter over the full body and any grasped object, treated as an extension of the final robot link. To enable real-time deployment while maintaining fine spatial resolution in critical task regions, VLPSA combines dual resolutions of this PSF using Boolean CBF compositions. We evaluate VLPSA on SafeLIBERO against safety-filtering baselines, where it achieves the highest collision avoidance rate among the evaluated methods, increasing collision avoidance from $23.1\%$ for the base $\pi_{0.5}$ policy to $91.2\%$ while surpassing its task success rate. We further deploy VLPSA on a Franka FR3 in cluttered scenes with dynamic obstacles and human interference, demonstrating real-time full-body safety during manipulation tasks.

\end{abstract}
\section{Introduction}

Vision-language-action (VLA) models are advancing general purpose robotic autonomy, with a growing ability to generalize across tasks and complex environments.  Popular models such as $\pi_{0.5}$ \cite{intelligence2025pi05visionlanguageactionmodelopenworld} and OpenVLA \cite{kim2024openvlaopensourcevisionlanguageactionmodel} demonstrate manipulation policies capable of diverse, long-horizon tasks. Despite this performance, learned polices do not provide formal safety guarantees, leaving a critical gap between their demonstrated capabilities and reliable deployment alongside humans and in unstructured environments. Closing this gap requires deciding how and where safety enters an increasingly capable learned policy. This question has been approached from several directions \cite{thakker2025riskguideddiffusiondeployingrobot} without clear consensus, particularly for contact-rich manipulation tasks.  VLA policies produce \textit{action} chunks at inference rate that is below the robot's desired control rate. A safety mechanism acting only on chunks is limited in its ability to respond to obstacles entering the workspace, which is a particular regime of interest for manipulation tasks executed alongside humans. This motivates enforcing safety at the level of individual actions, rather than longer-horizon action chunks.

Existing approaches primarily train safety into the policy or via a learned safety representation. Policy-level learning approaches include training on safe demonstration data \cite{cui2026liberosafetycomprehensivebenchmarkphysical} and applying constrained policy optimization \cite{zhang2026safevlasafetyalignmentvisionlanguageaction}. Safety can further be learned as a separate constraint, with latent safety filters offering a method to synthesize Hamilton-Jacobi reachability value functions \cite{bansal2017hamiltonjacobireachabilitybriefoverview} over world-model representations \cite{Nakamura_2025}. These approaches allow constraints without analytic specifications to be enforced. However, the safety behaviors depend on the fidelity of the learned policies or world models \cite{seo2025uncertaintyawarelatentsafetyfilters}, and thus do not provide independently specified geometric guarantees for previously unseen configurations. In contrast, we enforce safety constraints at deployment, at a control rate, over a geometric representation of the environment. Our method requires neither policy modification nor safety-specific data to be learned, while supplying formal guarantees of safety. 


\begin{figure}[t]
   \centering
   \includegraphics[width=.99 \columnwidth]{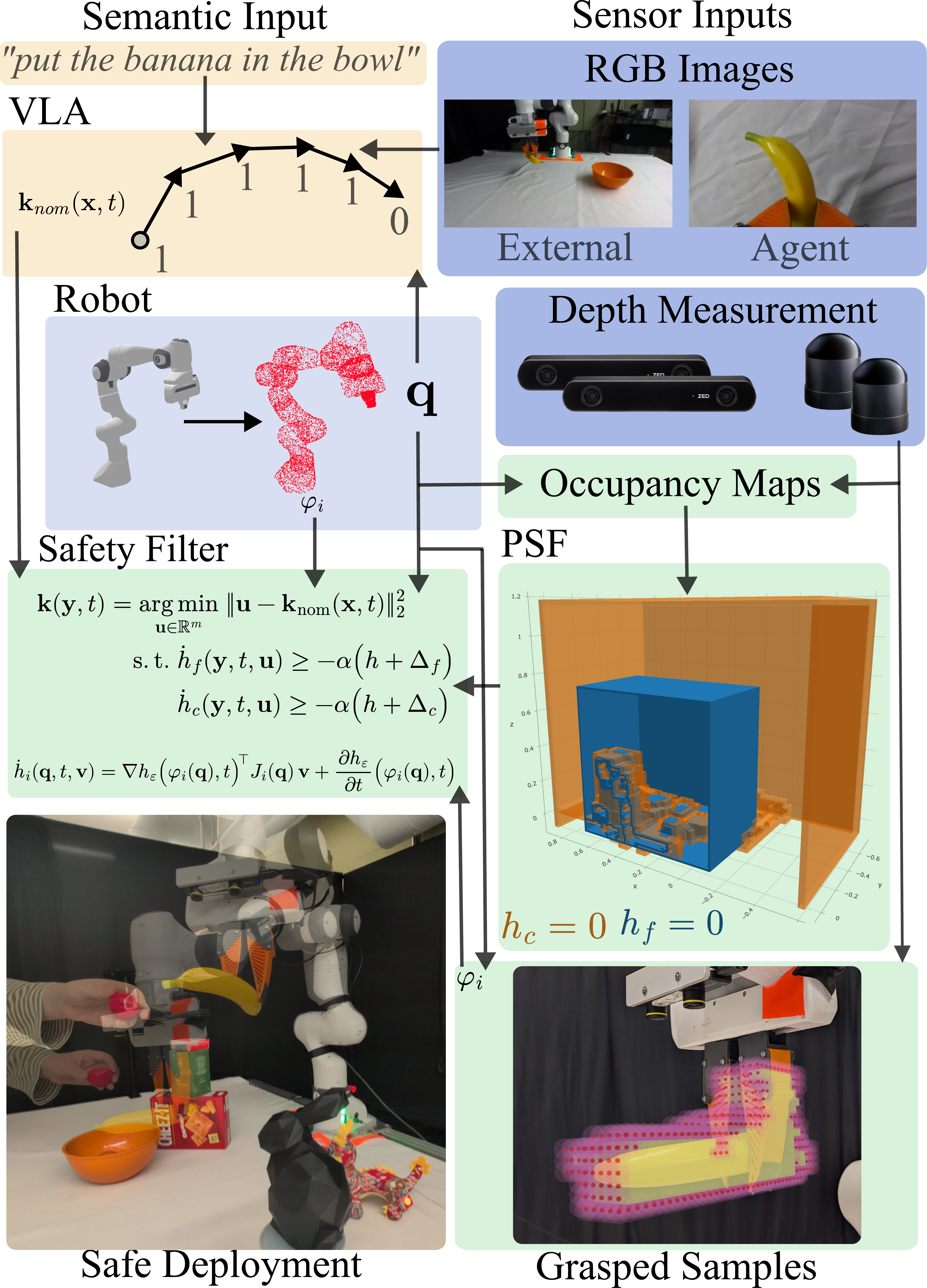}
    \caption{Pipeline for VLPSA: Sensory and semantic input and robot state are sent to the VLA for a nominal trajectory and gripper state. Depth sensors are fused with robot state information to create an occupancy map for the synthesis of PSFs at varying resolutions for safety filtering. When grasping, live samples are generated and updated to protect the held object, incorporated into the same safety filter.}
    \label{fig:hero}
\vspace{-7mm}
\end{figure}


Safety filters based on Control Barrier Functions (CBFs) provide a principled approach to safety-critical control \cite{ames_2017, ADA-SC-ME-GN-KS-PT:19}. They enforce forward invariance of a defined safe set while providing a modular method for minimally modifying nominal commands. Vision-Language-Safe Action (VLSA) applies safety filters to VLA policies, substantially improving collision avoidance, and in some cases, task performance \cite{hu2025vlsa}. However, the collision avoidance is limited to the end effector in static scenes. Extensions to dynamic obstacles have been shown in simulation \cite{park2026modelknowsattentionguidedsafety}, but full-body collision avoidance of VLA policies with 
moving obstacles has yet to be demonstrated on hardware. Held objects have been incorporated into CBF-based avoidance for deformable co-manipulation \cite{Aksoy_2026}, though as object-centric constraints under hand-designed controllers rather than as incorporated into the same full-body safety representation used for the robot.

Poisson Safety Functions (PSFs) \cite{bahati2025dynamic} provide a method to synthesize CBFs online.  One solves an elliptic partial differential equation (PDE), Poisson's equation, over a domain constructed 
from perception data to yield a functional and dynamic safety representation of the environment 
\cite{bena2025geometry}. Recent developments incorporate full-body safety guarantees for robot manipulators via sampling and buffering in dynamic scenes \cite{wilkinson2026fullbodydynamicsafetyrobot}. This formulation assumes a fixed body geometry. Manipulator's effective geometry changes when an object is grasped, so a sample set defined over the links alone ceases to cover the system mid-task. PSFs also introduce a tradeoff between spatial resolution of the safety representation and the cost of the PDE solve: cluttered scenes demand high resolution, real-time operation favors coarse grids.

We present \textit{Vision-Language-Poisson-Safe Actions} (VLPSA), a unified safety framework applied to learned polices providing conditions for guaranteed dynamic collision avoidance for both the full-robot body and its grasped objects. These guarantees are geometric and conditioned on the perception pipeline supplying a conservative bound on obstacle geometry; we characterize sensitivity to this assumption in Section~\ref{Hardware}. 

Our contributions are as follows: 

\setlength{\leftmargini}{1em}
\begin{itemize}
    \item Full-body manipulator PSF-CBF safety filter for VLA policies in dynamic, cluttered environments that safely filters the action chunk with no retraining of the nominal policy needed.
    \item A multigrid PSF formulation composing a coarse grid over the reachable workspace with a nested fine grid over the task region, combined through a Boolean `OR' composition \cite{Ong_2025}. 
    \item Extension of the robot representation to include unseen \textit{a priori} grasped objects. 
    \item Hardware validation on a Franka FR3 running a fine-tuned $\pi_{0.5}$ policy, in cluttered scenes and with obstacles entering the workspace during execution, paired with evaluation against filtering baselines on SafeLIBERO.
\end{itemize}

\section{Preliminaries}
\subsection{Control Barrier Functions}
We consider a non-linear control affine system,
\begin{equation}\label{affine}
    \dot{\bx}  = \bf(\bx) + \bg(\bx)\bu,
\end{equation}
with state $\bx\in \mathbb{R}^n$ and control input $\bu\in\mathbb{R}^m$, where $\bf: \mathbb{R}^n \rightarrow \mathbb{R}^n$  and $\bg: \mathbb{R}^n \rightarrow \mathbb{R}^{n\times m}$ are both locally Lipschitz. A locally Lipschitz control law $\bk: \mathbb{R}^n \times \mathbb{R}_+ \rightarrow \mathbb{R}^m$ renders a closed-loop system that admits a unique solution $t \mapsto \bx(t)$; we assume a solution exists $\forall t\geq 0$ \cite{perko2013differential}. A safe set can be defined as the unoccupied regions in the workspace. We encode the zero-super level set by a continuously differentiable function $h: \mathbb{R}^n \times \mathbb{R}_+ \rightarrow \mathbb{R}$,
\[ \mathcal{C}_t = \{ \bx\in\mathbb{R}^n | h(\bx,t) \geq 0 \}.\]
We seek a control input that can render the safe set forward invariant: given an initial condition $\bx(0) \in \mathcal{C}_0$, the controlled system
remains in the safe set $\bx(t) \in \mathcal{C}_t$ for all $t >0$.
\begin{definition}[Control Barrier Functions \cite{ADA-XX-JWG-PT:17}]

Let $\mathcal{C}_t$ be the time-varying $0$-superlevel set of a
continuously differentiable $h:\mathbb{R}^n\!\times\! \mathbb{R}_+\to\mathbb{R}$ satisfying
$\nabla h(\bx,t)\neq 0$ whenever $h(\bx,t)\!=\!0$. Then $h$ is a
time-varying Control Barrier Function for \eqref{affine}
if there exists $\alpha\!\in\!\mathcal{K}^e_\infty$ such that for all
$(x,t)\in\mathbb{R}^n\!\times\!\mathbb{R}_+$,
\begin{align}
\sup_{\bu\in\mathbb{R}^m} \dot h(\bx,t,\bu) &\triangleq
\nabla h(\bx,t)\cdot\big(f(\bx)+g(\bx)\bu\big) + \tfrac{\partial h}{\partial t}(\bx,t) \nonumber\\
&\geq -\alpha\big(h(\bx,t)\big).
\label{eq:cbf}
\end{align}
\end{definition}

Satisfying \eqref{eq:cbf} point-wise renders $\mathcal{C}_t$ forward
invariant \cite{ADA-SC-ME-GN-KS-PT:19}. Given a nominal controller $k_\mathrm{nom}$,
safety is classically enforced by the CBF-QP, which returns the
minimum-norm modification of $k_\mathrm{nom}$ satisfying \eqref{eq:cbf}.
When a task is specified by a Control Lyapunov Function (CLF)
$V:\mathbb{R}^n\to\mathbb{R}_+$, stability and safety can be  combined in a single optimization problem.
Relaxing the CLF condition with a slack variable $\delta\geq 0$ preserves
feasibility when the two objectives conflict, yielding the CLF-CBF-QP
\cite{ames_2013_towards}:
\begin{equation}
\begin{aligned}
(\bu^*,\delta^*) = \underset{\bu\in\R^m,\,\delta\geq 0}{\arg\min}\quad
& \|\bu\|_2^2 + p\,\delta^2 \\
\text{s.t.}\quad
& \dot V(\bx,u) \leq -\gamma V(\bx) + \delta, \\
& \dot h(\bx,t,u) \geq -\alpha\big(h(\bx,t)\big),
\end{aligned}
\label{eq:clfcbfqp}
\end{equation}
where $\gamma>0$ sets the nominal convergence rate and $p>0$ penalizes
relaxation. Safety is enforced as a hard constraint, while tracking is
sacrificed only as needed.

\subsection{Poisson Safety Functions}

Definition~\ref{eq:cbf} requires a continuously differentiable $h$
with non-vanishing gradient on the safe set boundary. 
Poisson Safety Functions (PSFs) \cite{bahati2025dynamic} provide a systematic method to synthesize smooth CBFs $h_0$ directly from perception data.

\begin{definition}(Poisson Safety Function \cite{bahati2025dynamic})
Let $\Oc \subset \re^3$ denote the open, bounded, and connected free space with smooth boundary $\pOc$ corresponding to obstacle surfaces. We call $h_0:\re^3 \rightarrow \re$ a Poisson Safety Function if it is the unique solution to the Dirichlet problem for Poisson's equation:
\begin{gather}\label{eq: poisson's eq}
\!\!\!\left \{
    \begin{aligned}
        \frac{\partial^2 h_0}{\partial x^2}(\by) + \frac{\partial^2 h_0}{\partial y^2}(\by) + \frac{\partial^2 h_0}{\partial z^2}(\by) &= f(\by)&  \forall \by \in \Oc,\\
        h_0(\by) &= 0 &  \forall \by \in \pOc, \\
    \end{aligned}
    \right.
\end{gather}
where $f: \Oc \rightarrow \re_{<0}$ is the prescribed forcing function.  
\end{definition}

The Dirichlet boundary conditions enforce the barrier function's zero level set to lie on the obstacle surfaces while choosing $f<0$ forces $h_0 > 0$ in free space. Choosing a smooth forcing function $f \in C^{\infty}(\Omega)$ yields $h_0\in C^{\infty}(\Omega_c)$ \cite{gilbarg1977elliptic}.
As shown in \cite{bahati2025dynamic}, $h_0$ is a valid CBF for a single integrator system, $h_0 := h$, and can be further extended to higher order systems due to its smoothness. A single Poisson solution yields a globally smooth CBF for arbitrary environmental occupancy maps, defining the PSF. Considering a time-varying domain $\Oc = \Oc(t)$, we can extend Equation~\eqref{eq: poisson's eq} to moving obstacles, giving a moving boundary value problem solved online to yield $t\mapsto h_0(\by,t)$ \cite{bena2025geometry}.

\subsection{Application to Manipulators and VLA Policies}

Following \cite{SingletaryMolnarSafetyCriticalFood, wilkinson2026fullbodydynamicsafetyrobot}, we enforce safety of the manipulator at the kinematic level. That is we consider the system
\begin{equation}
    \dot{\bq} = \bv,
    \label{kinematic}
\end{equation}
with the commanded joint velocity as the control input $\bv \in \mathbb{R}^n$. Safety guarantees for \eqref{kinematic} extend to the full second-order manipulator dynamics provided the commanded velocity is sufficiently tracked by the low-level controller \cite{SingletaryMolnarSafetyCriticalFood}. 
Safety specifications for manipulators are naturally expressed in
task space, whereas the control input in \eqref{kinematic} acts at
the joint-velocity level. The two are related by forward kinematics:
for the $i$-th point on the robot body, let
$\varphi_i:\mathbb{R}^n\to\mathbb{R}^3$ denote its kinematic map, so
that $\by_i=\varphi_i(\bq)$ and,
$\dot{\by}_i\! =\! \tfrac{\partial\varphi_i}{\partial\bq}\dot{\bq}
\!=\! J_i(\bq)\bv$, with $J_i(\bq)\in\mathbb{R}^{3\times n}$ the corresponding Jacobian. Defining $h_i(\bq,t)\!\triangleq \!h(\varphi_i(\bq),t)$, the safety
condition becomes
\begin{equation}
\dot h_i(\bq,t,\bv)
= \nabla h\big(\varphi_i(\bq),t\big)^{\!\top} J_i(\bq)\,\bv
+ \frac{\partial h}{\partial t}\big(\varphi_i(\bq),t\big),
\label{eq:hdot}
\end{equation}
where $\nabla h$ is the gradient with respect to the
task-space argument. Note that \eqref{eq:hdot} is affine in $\bv$. We
assume $J_i(\bq)^{\!\top}\nabla h\! \neq\! 0$ whenever
$h_i(\bq,t)\!=\!0$, so that each $h_i$ is a valid CBF for
\eqref{kinematic}. Since the manipulator body is a continuum,
enforcing \eqref{eq:hdot} at every point yields a semi-infinite
program; instead we impose it pointwise on a finite sample set
$\mathcal{Y}(\bq)\!=\!\{\varphi_i(\bq)\}_{i=1}^{N}$, each point inducing
its own constraint:
\begin{equation}
\begin{aligned}
    (\bv_{\mathrm{safe}}, \delta^*) =
    \underset{\bv \in \re^n,\, \delta \geq 0}{\arg\min} \quad
    & \| \bv - \bv_{\mathrm{nom}} \|_2^2 + p\,\delta^2 \\
    \text{s.t.} \quad
    & \dot h_i(\bq,t,\bv) \geq -\alpha_i h_i(\bq,t) \\
    & \dot V(\bq,\bv) \leq -\gamma V(\bq) + \delta,
\end{aligned}
\tag{MC-CBF-QP}
\label{eq:mc-cbf-qp}
\end{equation}
where $\bv_{\mathrm{nom}}$ is the nominal joint velocity command,
$\alpha_i > 0$, and $\gamma, p > 0$ are as in \eqref{eq:clfcbfqp}. 

Enforcing safety at these finite points does not provide any safety guarantees for the full robot body. Prior work \cite{wilkinson2026fullbodydynamicsafetyrobot} combats this issue by sampling the robot surface at a fixed resolution $\varepsilon$ and buffering the occupancy map by that resolution. Solving Poisson's equation on this buffered domain gives the PSF  $h_\varepsilon(\mathbf{y})$. Satisfying the safety filter formulation of (\ref{eq:cbf}) with the buffered safety function renders the entire
robot surface collision-free with respect to the true environment \cite{wilkinson2026fullbodydynamicsafetyrobot}. Any
point set meeting the $\varepsilon$ bound inherits this safety guarantee, which we exploit in Section \ref{inject_Section}. 

It remains to specify the nominal control input. Let $\mathcal{O} = \mathcal{I}^{n_c} \times \mathcal{L} \times \mathbb{R}^n$ denote the observation space; $\mathcal{I}$ is the RGB input from $n_c$ camera images, $\mathcal{L}$ is the natural language instruction and $\mathbb{R}^n$ is the robot state. Then, a VLA policy is a map $\bm{\pi}:\ \mathcal{O} \rightarrow \mathbb{R}^{T\times m}$ such that a single observation $o_t \in \mathcal{O}$ is mapped to a size-$T$ action chunk $\ba_{1:T} = \bm{\pi}(o_t)$ with each action $\ba_k  \in \mathbb{R}^m$. Each action chunk is executed open-loop between fresh inference calls at a rate of $f_\pi$. The VLA policy is trained on demonstrations, and has no input representation of free space, nor mechanism of enforcing
safety constraints out of its training distribution. We provide a method to enforce full-body safety on a nominal policy in an arbitrary environment while preserving its task performance.

\color{black}
\section{Method}

In this section we present a method for safety filtering VLA policies using Poisson Safety Functions. Building on previous full-body safety work \cite{wilkinson2026fullbodydynamicsafetyrobot}, we apply a CBF-CLF safety filter to enforce safety constraints while tracking the VLA's desired trajectory. We first present the core safety-filtering architecture, then extend it to incorporate grasp-triggered sample injections to keep objects the robot has picked up safe. Finally, we introduce a multigrid Poisson solve, deployed in our hardware experiments, to achieve both precision and speed in cluttered, dynamic environments.

\subsection{VLA + Safety Filter}

Safety is integrated directly over the VLA's action output to ensure full-body safety during deployment. Filtering each action independently risks compounding the tracking error and losing the VLA's semantic intent as filtered actions diverge from the policy's plan \cite{romer2025pacs}. Instead, we treat the VLA as a high-level planner; given visual and language inputs, the policy outputs an action chunk $\mathbf{a}_{1:T}$ over a $T-$step horizon. We interpret this as a desired continuous trajectory $\bold{x}_d(t)$, obtained by linearly interpolating between the way-points of the action chunks, which is tracked by a low-level safety filter operating at a faster control rate.

The tracking error is defined
\[
\be = \begin{bmatrix}
    \bx_d - \bx_{ee} \\ \text{rpy}_{\text{err}} \ 
\end{bmatrix} \in \mathbb{R}^6.
\]
The variable $\text{rpy}_{\text{err}}$ represents the wrapped angular error between the desired and current orientation and $\bx_{ee} \in \mathbb{R}^3 $ is the current end-effector position. We define $V(e) = \frac{1}{2}e^TWe,$ with $ W\succ0$ required for a valid CLF.

Treating the reference trajectory as quasi-static, the tracking error derivative simplifies to $\dot{V} \!= \!- \be^TW^TJ_{ee}(\bq)\bu $, where $J_{ee}(\bq)$ is the end-effector Jacobian and $\bu\in\mathbb{R}^7$ is the commanded joint velocity. In the absence of safety considerations and kinematic limits, this condition drives the tracking error to zero. Strict convergence can conflict with safety constraints, hence we execute it slacked as in the CLF-CBF QP (\ref{eq:clfcbfqp}).

\subsection{Grasp-triggered Sample Injection}\label{inject_Section}

Full body safety, in the context of manipulation, requires consideration of grasped objects to ensure complete collision avoidance. An item held by the robot  may collide with the environment, and thus, it should affect the safety constraints in (\ref{eq:mc-cbf-qp}). 
Existing learning-based safety filters utilize configuration-space representations \cite{long2026neuralconfigurationspacebarriersmanipulation,chen2026cssdfnetsafemotionplanning}, which incorporate geometry during training or representation construction, making dynamic changes in the manipulated geometry difficult to accommodate \cite{hu2025vlsa}. 
Geometric safety filters can dynamically account for grasped objects by adding bounding spheres \cite{morton2025safetaskconsistentmanipulationoperational}; these representations can be substantially over-approximated. In cluttered environments, this approximation may classify collision-free motions as unsafe.

The key advantage of a sample-based safety representation is that samples are used to represent generic geometric configurations and not tailored to specific robot geometry. 
%
%
A rigidly grasped object can be incorporated into the sample-based
representation by augmenting the robot's sample set without changing the controller's formulation. 
Let $t_g$ define the time at which grasping onset is observed and let $\mathcal{C}(t)\subset\Omega$ denote the occupied set of the environment
representation at time $t$. Rather than a tight pose estimate, we only require a bounding volume of the grasped object: a compact set $\hat{\mathcal{B}}_O^E$, expressed in end-effector coordinates, containing the object.

\begin{proposition}[Safety of a Grasped Object]
\label{cor:injection}
Suppose that $\forall t\ge t_g$ the object remains rigidly fixed to the
end-effector $\{E\}$, and is contained in a fixed compact bounding volume $\hat{\mathcal{B}}_O^E$ expressed in end-effector coordinates.
Let $O^E=\{o_1^E,\dots,o_M^E\}$ be an $\varepsilon_O$-cover
of $\partial\hat{\mathcal{B}}_O^E$ with $\varepsilon_O\le\varepsilon$, define
\[
o_i(\mathbf{q})=p_E(\mathbf{q})+R_E(\mathbf{q})\,o_i^E,
\qquad
\mathcal{Y}_{\mathrm{aug}}(\mathbf{q})=\mathcal{Y}(q)\cup O(\mathbf{q}).
\]
 Assume $\hat{\mathcal{B}}_O(\bq(t_g))\cap\mathcal{C}(t_g)=\emptyset$. If \eqref{eq:mc-cbf-qp} is enforced for every
$y\in\mathcal{Y}_{\mathrm{aug}}$, then the robot surface and the bounding volume $\hat{\mathcal B}_O(\mathbf{q})$ are collision free with respect to the true environment, and hence so is the true object for all $t\ge t_g$.
\end{proposition}

\begin{proof}
Rigid motions are isometries, so $O(\bq)$ remains an $\varepsilon_O$-cover of $\partial\hat{\mathcal{B}}_O(\bq)$ for every $q$. Since $\mathcal{Y}(\mathbf{q})$ satisfies the $\varepsilon$-sampling condition by construction and $\varepsilon_O\!\leq\!\varepsilon$, the augmented set $\mathcal{Y}_{\mathrm{aug}}(q)$ satisfies the sampling condition required by Theorem~1 in \cite{wilkinson2026fullbodydynamicsafetyrobot}. This establishes forward invariance of the safe set for the robot and the
represented object geometry. Since $\hat{\mathcal{B}}_O(\bq(t_g))$ is initially
free of occupancy, invariance precludes any subsequent intersection with
$\mathcal{C}$, and containment of the object in $\hat{\mathcal{B}}_O$ transfers
the conclusion to the object itself.
\end{proof} 

\begin{remark}[Practical Initial Conditions]
At pick up, objects are inherently resting in contact and unsafe, thus not in the forward invariant set. Enforcing CBF conditions as in \ref{eq:mc-cbf-qp} provides convergence to the safe set \cite{ADA-SC-ME-GN-KS-PT:19}.
\end{remark}

\begin{figure}[t]
    \centering
    \includegraphics[width=\columnwidth]{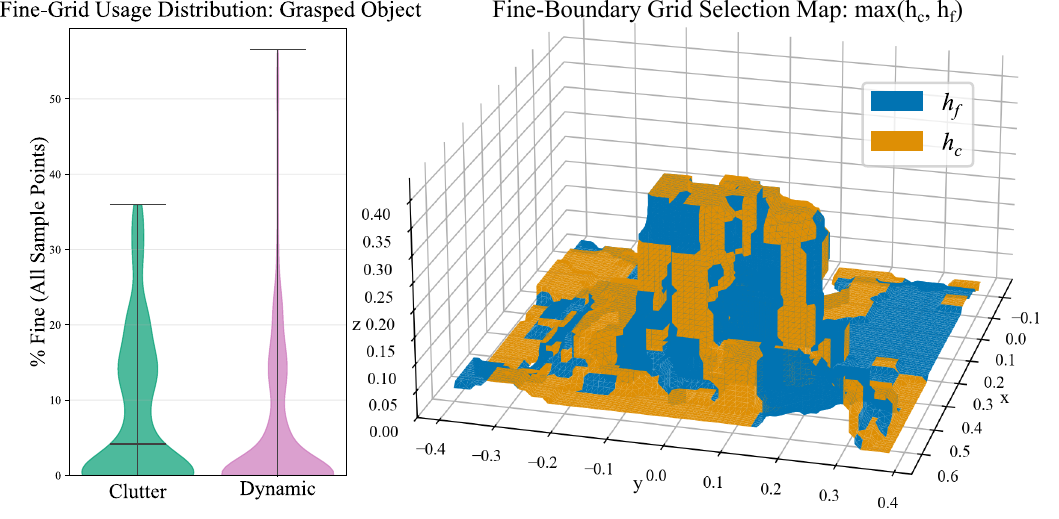}
    \caption{Multigrid Usage. \textbf{(Left)} Fine-grid usage (\% of sample points) for the grasped object during a trial in static clutter and with dynamic obstacles. The grasped object shows consistent moderate usage under clutter but sparser, higher peaks under moving obstacles. \textbf{(Right)} Zero-level set of the fine PSF $h_f$ in a cluttered environment showing which grid wins at each point.}
    \label{fig:finevscoarse}
\vspace{-5mm}
\end{figure}

\subsection{Multi-resolution Poisson} \label{multigrid}

The proposed safety filter is formulated in continuous space; but implementing Poisson Safety Functions requires discretizing the workspace into an occupancy grid. Discretization introduces a tradeoff between speed and grid resolution. A fine resolution gives a precise occupancy grid but causes a significant increase in solve time. Fast solve times are required for moving obstacles, but for clutter high resolution is key. On hardware, mitigating this trade off is prohibitively computationally expensive.

We therefore construct two safety functions at different resolutions:
a coarse function over a large domain $\Omega_c \subset \re^3$ with
grid spacing $\delta_c$, and a fine function over a nested region
$\Omega_f \subseteq \Omega_c$ with $\delta_f < \delta_c$, placed over
the task region. Occupancy is transferred conservatively: a coarse
cell is marked occupied if any fine cell within is occupied.
Writing $\mathcal{O}_i$ and $\mathcal{F}_i$ for the occupied and free
regions at level $i$, this gives $\mathcal{O}_f \subseteq
\mathcal{O}_c$, and hence $\mathcal{F}_c \subseteq \mathcal{F}_f$ on
$\Omega_f$. Solving Equation (\ref{eq: poisson's eq}) on each domain yields $h_c(\by)$ and
$h_f(\by)$. The fine solve imposes zero Dirichlet boundary conditions on the artificial region boundary $\partial \Omega_f$, thus $h_f$ vanishes there even where the workspace is free; switching to $h_f$ inside the region would therefore report spurious unsafety at the grid walls and forfeit the continuity that Poisson safety functions provide. We instead compose the two levels using the Boolean OR construction of \cite{Ong_2025}, letting
\begin{equation}
  h(\by) = \max\big( h_c(\by),\, h_f(\by) \big),
  \label{eq:or-composition}
\end{equation}
and defining $\Delta_i(\by) \triangleq | h(\by) - h_i(\by) |$
for $i \in \{f, c\}$. Rather than differentiating the nonsmooth
maximum, we impose one constraint per level, each relaxed by its own
difference, resulting in the safety filter
\begin{equation}
\begin{aligned}
  \bk(\by,t) = \argmin_{\bu \in \re^m} \;
    & \|\bu - \bk_{\mathrm{nom}}(\bx,t)\|_2^2 \\
  \mathrm{s.t.} \;
    & \dot h_f(\by,t,\bu) \geq -\alpha\big( h + \Delta_f \big) \\
    & \dot h_c(\by,t,\bu) \geq -\alpha\big( h + \Delta_c \big),
\end{aligned}
\tag{CBF-QP} \label{eq:safety-filter}
\end{equation}
enforced at every query point $\by \in \mathcal{Y}_{\mathrm{aug}}(q)$
in the augmented sample set, yielding $2|\mathcal{Y}_{\mathrm{aug}}|$ total constraints. Whichever level attains the maximum has
$\Delta_i = 0$ and recovers the nominal CBF condition, while the
inactive level is slackened in proportion to its gap. Each $\Delta_i$ decays continuously to zero as the levels cross, hence the filtered input varies continuously throughout the transition, and the coarse level certifies safety near the fine window's boundary without
any explicit mode switch. Figure \ref{fig:finevscoarse} demonstrates where the fine grid safety constraint overrides the coarse. The number of constraints is doubled in the MC-CBF-QP, however solve time is still under $2$ms for $2$cm sample points, with the Poisson solve time under $5$ms (Table \ref{tab:solve-scaling}).

\begin{table}[t]
  \centering
  \caption{Solve cost  for a fixed workspace extent (\SI{2.4}{\meter} per side)
    and a fixed \num{100}-iteration budget. Mean $\pm$ std over $27$--$1708$
    solves per configuration on an RTX 4090, with
    $|\mathcal{Y}_\mathrm{aug}| = 679$. The total time includes all numerical processing. The multigrid scheme resolves
    \SI{1.0}{\centi\meter} over the region of interest at coarse-grid cost;
    it doubles the number of QP constraints, but latency is dominated
    by the Poisson solve.}
  \label{tab:solve-scaling}
  \footnotesize
  \begin{tabular*}{\columnwidth}{@{\extracolsep{\fill}}l c c c@{}}
    \toprule
    Grid ($\varepsilon$, $N^3$) & Poisson [ms] & QP [ms] & Total [ms] \\
    \midrule
    \multicolumn{4}{@{}l}{\textit{Single uniform}} \\
    \SI{3.0}{\centi\meter}, $80^3$  & $3.44 \pm 0.23$   & $0.72 \pm 0.02$ & $6.5 \pm 1.9$ \\
    \SI{2.0}{\centi\meter}, $120^3$ & $9.12 \pm 0.44$   & $0.76 \pm 0.04$ & $23.1 \pm 6.6$ \\
    \SI{1.5}{\centi\meter}, $160^3$ & $20.20 \pm 0.46$  & $0.78 \pm 0.12$ & $38.5 \pm 5.7$ \\
    \SI{1.0}{\centi\meter}, $240^3$ & $91.80 \pm 0.95$  & $0.61 \pm 0.07$ & $261.6 \pm 9.0$ \\
    \SI{0.5}{\centi\meter}, $480^3$ & $724.19 \pm 3.71$ & $0.59 \pm 0.03$ & $2059 \pm 33$ \\
    \midrule
    \multicolumn{4}{@{}l}{\textit{Proposed (multigrid)}} \\
    \SI{3.0}{}/\SI{1.0}{\centi\meter}, $80^3{+}50^3$
      & $4.75 \pm 0.48$ & $1.35 \pm 0.19$ & $8.5 \pm 1.9$ \\
    \bottomrule
  \end{tabular*}
\vspace{-3mm}
\end{table}

\section{Results}

\subsection{Simulation Benchmark}

\begin{table}[t]
\centering
\caption{VLPSA vs.\ VLSA$^*$~\cite{hu2025vlsa} (full action space) and
nominal $\pi_{0.5}$ on SafeLIBERO (Levels I+II).
VLPSA-NI ablates injection of held-object samples into the safe set.
Success (SR), Avoidance (CAR, $100-$ collision rate), and Safe Success (SSR)
rates in \%; Steps is the mean number of steps to reach success, over
successful episodes. $\pm$ denotes a 95\% confidence interval (normal
approximation; $n{=}400$ episodes per suite). Bold marks the best method where
its interval does not overlap that of the runner-up.}
\label{tab:comparison}
\small
\begin{tabular*}{\linewidth}{@{\extracolsep{\fill}}l r r r r@{}}
\toprule
Method & SR\,$\uparrow$ & CAR\,$\uparrow$ & SSR\,$\uparrow$ & Steps\,$\downarrow$ \\
\midrule
\multicolumn{5}{@{}l}{\emph{Spatial}}\\
Nominal   & 57.8\,$\pm$\,4.8          & 20.8\,$\pm$\,4.0          & 19.2\,$\pm$\,3.9          & 119.5\,$\pm$\,6.3 \\
VLSA$^*$  & 73.8\,$\pm$\,4.3          & 70.5\,$\pm$\,4.5          & 59.8\,$\pm$\,4.8          & 133.4\,$\pm$\,4.2 \\
VLPSA-NI  & 71.5\,$\pm$\,4.4          & 75.8\,$\pm$\,4.2          & 54.5\,$\pm$\,4.9          & 123.0\,$\pm$\,5.9 \\
VLPSA     & 72.2\,$\pm$\,4.4          & \textbf{89.8}\,$\pm$\,3.0 & 66.5\,$\pm$\,4.6          & 136.9\,$\pm$\,5.9 \\
\addlinespace
\multicolumn{5}{@{}l}{\emph{Goal}}\\
Nominal   & 49.8\,$\pm$\,4.9          & 27.8\,$\pm$\,4.4          & 20.0\,$\pm$\,3.9          & \textbf{116.9}\,$\pm$\,7.5 \\
VLSA$^*$  & \textbf{85.2}\,$\pm$\,3.5 & 88.2\,$\pm$\,3.2          & \textbf{76.8}\,$\pm$\,4.1 & 134.7\,$\pm$\,4.9 \\
VLPSA-NI  & 55.5\,$\pm$\,4.9          & 82.2\,$\pm$\,3.7          & 44.5\,$\pm$\,4.9          & 132.4\,$\pm$\,7.5 \\
VLPSA     & 57.8\,$\pm$\,4.8          & 86.8\,$\pm$\,3.3          & 53.2\,$\pm$\,4.9          & 147.5\,$\pm$\,7.2 \\
\addlinespace
\multicolumn{5}{@{}l}{\emph{Object}}\\
Nominal   & 49.5\,$\pm$\,4.9          & 26.8\,$\pm$\,4.3          & 20.8\,$\pm$\,4.0          & 143.0\,$\pm$\,6.7 \\
VLSA$^*$  & \textbf{68.2}\,$\pm$\,4.6 & 73.5\,$\pm$\,4.3          & \textbf{55.0}\,$\pm$\,4.9 & 170.1\,$\pm$\,5.0 \\
VLPSA-NI  & 38.0\,$\pm$\,4.8          & 79.8\,$\pm$\,3.9          & 25.2\,$\pm$\,4.3          & 154.3\,$\pm$\,9.2 \\
VLPSA     & 39.0\,$\pm$\,4.8          & \textbf{94.2}\,$\pm$\,2.3 & 38.2\,$\pm$\,4.8          & 148.6\,$\pm$\,7.9 \\
\addlinespace
\multicolumn{5}{@{}l}{\emph{Long}}\\
Nominal   & 39.5\,$\pm$\,4.8          & 17.0\,$\pm$\,3.7          & 13.2\,$\pm$\,3.3          & 297.9\,$\pm$\,14.2 \\
VLSA$^*$  & 45.0\,$\pm$\,4.9          & 59.8\,$\pm$\,4.8          & 26.0\,$\pm$\,4.3          & 343.7\,$\pm$\,11.8 \\
VLPSA-NI  & 43.8\,$\pm$\,4.9          & 69.0\,$\pm$\,4.5          & 27.8\,$\pm$\,4.4          & 267.7\,$\pm$\,11.8 \\
VLPSA     & 53.0\,$\pm$\,4.9          & \textbf{94.0}\,$\pm$\,2.3 & \textbf{50.2}\,$\pm$\,4.9 & 273.3\,$\pm$\,11.3 \\
\midrule
\multicolumn{5}{@{}l}{\emph{Overall}}\\
Nominal   & 49.1\,$\pm$\,2.4          & 23.1\,$\pm$\,2.1          & 18.3\,$\pm$\,1.9          & 160.6\,$\pm$\,6.5 \\
VLSA$^*$  & \textbf{68.1}\,$\pm$\,2.3 & 73.0\,$\pm$\,2.2          & 54.4\,$\pm$\,2.4          & 177.8\,$\pm$\,5.4 \\
VLPSA-NI  & 52.2\,$\pm$\,2.4          & 76.7\,$\pm$\,2.1          & 38.0\,$\pm$\,2.4          & 161.5\,$\pm$\,5.6 \\
VLPSA     & 55.5\,$\pm$\,2.4          & \textbf{91.2}\,$\pm$\,1.4 & 52.1\,$\pm$\,2.4          & 174.3\,$\pm$\,5.5 \\
\bottomrule
\end{tabular*}
\vspace{-5mm}
\end{table}

\color{black}
Simulation benchmarks are used for analysis of safety-filtering capabilities. SafeLIBERO~\cite{hu2025vlsa} augments LIBERO task suites~\cite{liu2023libero} with obstacle-rich variants across two difficulty levels (16 tasks, 32 scenarios, 50 episodes each). The suites isolate distinct sources of variation: \textit{Spatial} (varied layout), \textit{Object} (varied target object), \textit{Goal} (varied task command), and \textit{Long} (chained subgoals). Level~I places obstacles near target objects, whereas Level~II places obstacles along the transport path. Level~II challenges end-effector-only filters, as safety bounds must continuously protect the geometry of the held object during transit. Table \ref{tab:comparison} shows the combined results.

Fine-tuning VLAs on obstacle-rich data is costly and insufficient for hard safety guarantees, as safety degrades under spatial reconfigurations or unseen obstacle geometries~\cite{cui2026liberosafetycomprehensivebenchmarkphysical}. Conversely, zero-shot models perform poorly; for instance, the LIBERO-finetuned $\pi_{0.5}$ checkpoint~\cite{intelligence2025pi05visionlanguageactionmodelopenworld} exhibits high collision rates when deployed on SafeLIBERO (Table~\ref{tab:comparison}).

VLSA~\cite{hu2025vlsa} serves as the primary baseline safety filter. It prompts a vision-language model to identify the obstacle most probable to cause a collision, grounds it in 3D from RGB-D, and fits a minimum-volume enclosing ellipsoid to enforce end-effector collision avoidance. We denote the authors' implementation as VLSA*, re-evaluated under our standardized protocol; its metrics reflect direct comparability across Table~\ref{tab:comparison} rather than originally reported values. While unreleased method implementations assign per-obstacle ellipsoids~\cite{wang2026vlaknowslimitsadaptive}, multi-obstacle constraints can scale poorly and require complex relaxations~\cite{molnar2023composingcontrolbarrierfunctions}. Furthermore, single-ellipsoid abstractions over-approximate concave geometries, causing admissible control sets to collapse. No prior work addresses full-body collision avoidance with held-object protection, which we tackle using PSFs~\cite{wilkinson2026fullbodydynamicsafetyrobot}.

All methods in Table~\ref{tab:comparison} are evaluated under a single protocol that updates the original VLSA evaluation in four key respects. An expanded collision  detection scheme scans all simulation contacts and distinguishes visual mesh intersections from MuJoCo \cite{todorov2012mujoco} collision geometry artifacts. We slightly modify the success criteria parameters to align the evaluation metrics with qualitative semantic task completion. To avoid token use, we replace online vision language model (VLM) obstacle identification with ground-truth labels. Finally, all controllers run at 100 Hz. While absolute metrics differ from the prior literature due to these updates, cross-method comparisons in Table~\ref{tab:comparison} remain self-consistent.
\color{black}

\subsection{Simulation Discussion}
Safety filtering markedly improves safety and benefits performance, but there are trade offs. Note that safe success rates do not meaningfully differ between VLSA* and VLPSA overall, but VLPSA results in \textbf{three times fewer collisions}. Losing unsafe success is a cost to conservative safety filtering, but collision avoidance is a necessary trade off for many deployments. Additionally, VLSA* and VLPSA both meaningfully improve raw success rates, showing that in aggregate there is not necessarily a success cost to filtering. In the object suite, while VLPSA reduces the success rate from nominal, the safe success rate nearly doubles. Qualitatively, failures arise when the VLA is filtered off its path and attempts to pick up an object it is filtered from touching, showing a vision-language failure but highlighting the forward invariance of the safe set under VLPSA deployment. Some unnecessary failures occur, particularly with an object in hand, when a long horizon plan would be required while the VLA's short term goal is fundamentally unsafe. 

A note on the fallibility of simulation is warranted. MuJoCo \cite{todorov2012mujoco} resolves contact against simplified collision geometries that can extend beyond the
rendered visual mesh.  Hence, contacts can be reported while objects remain
visibly separated. Filtering this does not preclude contact from being recorded, 89\% of VLA collisions cannot be rejected, having real penetration, versus only 41\% of VLPSA collisions. This clear discrepancy shows this is a differentiating metric, while not perfect. All analysis therefore counts only \emph{unrejectable}
collisions: those in which the visual meshes penetrate, or in which MuJoCo
records no contact pair yet the obstacle is displaced. The latter case accounts
for fewer than 1 in 10 collisions tracked. While we consider the unrejectable-collision count the more faithful metric, we note that the set of statistically separated results in Table \ref{tab:comparison} is unchanged under the
original 1\,mm displacement criterion: every bolded entry in
Table~\ref{tab:comparison} remains bolded, and no new one appears. The leading method changes in two cells, Goal CAR and overall SSR, both moving in VLPSA's favor; in neither case does the difference exceed the confidence intervals.

A full-body safety filter with complete environment coverage should, in
principle, produce no unrejectable collisions. Three effects prevent this in
practice. Perception, even the noiseless perception available in simulation,
cannot protect against obstacles that are never observed, nor against a grasped
object that slips substantially within the gripper after the grasp is
established. The third effect is a simulation artifact: contact with collision geometry,
which the filter neither prevents nor is designed to prevent, engages contact
dynamics that invalidate the velocity-tracking assumption upon which the filter
rests, and the resulting tracking error can carry the visual mesh into a
penetration.

Much of the benefit of VLPSA is erased without accounting for the held object (VLPSA-NI) as shown in Table \ref{tab:comparison}, and qualitatively from scene construction, end effector collisions are mostly what stands to be avoided. Dynamic full body collision avoidance is practical using our method, which is not fully exercised in this benchmark. We therefore weight hardware deployments to establish the method's practicality. 

\subsection{Hardware Experiments}\label{Hardware}

We validate the method on a Franka FR3 to demonstrate that the filter
transfers to real hardware. The safe set is constructed from real
perception data rather than simulated, and the Poisson solve must keep
pace with the policy at deployment rates. The experiments
here establish that the full perception-to-filter pipeline closes the loop on
physical hardware, including dynamic obstacle avoidance.

\begin{figure*}[t]
  \centering
  \includegraphics[width=\textwidth]{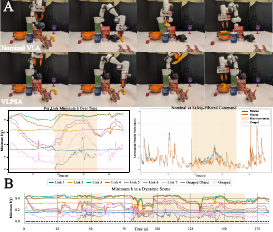}
  \caption{ (A) Hardware validation of the VLA policy with and without the proposed safety layer in a cluttered tabletop environment for the task \textit{``Pick up the banana.''} \textbf{(Top)} The cluttered environment is out-of-distribution for the nominal VLA policy, which completes the task but collides with multiple objects in the scene. \textbf{(Middle)} With the proposed full-body safety architecture, the VLA policy completes the same task while avoiding all collisions. \textbf{(Bottom Left)} The minimum active Poisson Safety Function value per link remains positive through-out the trial. \textbf{(Bottom Right)} The nominal trajectory versus the safety filtered trajectory throughout the trial. (B) The minimum value of the PSF per-link throughout a long horizon trial with the same text prompt as above. This dictates two successful executions of the task in clutter with dynamical obstacles.}
\label{fig:clutter-strip}
\vspace{-5mm}
\end{figure*}

\subsubsection{Perception to Occupancy Map} \label{sec:perception}
To deploy safety filtering on hardware, a dynamic occupancy map is required for the multi-resolution Poisson solution (Section \ref{multigrid}).
Following OctoMap \cite{hornung13auro}, we track per-voxel occupancy in
log-odds, incrementing on a sensor hit and decrementing when a ray passes
through, and mark a cell occupied above a probability threshold. Because the
base is fixed and the Poisson solver requires a uniform grid, we forgo the
octree and define occupancy on a $1.5 \times 1.5 \times 1.5$ m grid at 1 cm
resolution. Maps are built from two Hesai JT128
LiDARs and two ZED2i stereo cameras; parameters are given in the supplementary
code.

Feasible collision avoidance requires excluding the arm itself. We maintain a
\textit{robot mask} from the sampled surface points used in the safety filter
QP (Eqn. \ref{eq:mc-cbf-qp}), unioning $\varepsilon_E$ balls about each; any hit
within $\varepsilon_E$ of a sample is attributed to the robot and discarded. The
exclusion is inflated to $r_E = \varepsilon_E + 2$ cm to absorb arm motion
during a LiDAR sweep, extrinsic calibration drift, and depth noise. Provided no
obstacle initializes inside this buffer and the robot keeps its safe distance,
self-filtering discards no information. \textit{Flying pixels} and depth
uncertainty \cite{10.1145/1921264.1921293} are most damaging when they mark the
body or its immediate halo as occupied, producing infeasibility or self-evasion.
Per sensor, we therefore ignore LiDAR hits inside the occlusion shadow each
robot sample casts at the current pose, and camera hits inside an anisotropic
ellipse about each sample's projected pixel, elongated along the visual ray:
depth images carry error in range but not in pixel location, so a sample's
$(u,v)$ projection is exact while its measured depth is not. Exclusions are per sensor, so no region is blinded:
obstacles observed by another sensor are still respected.

Rather than raytrace, we forward-project: voxel centers transform into each
camera frame as $(u,v,d)$ and each LiDAR frame as $(\theta,\phi,r)$, both single
matrix operations. A voxel is cleared when all neighboring pixels or scans report
surfaces more than $\varepsilon_c$ beyond its depth, standing in for a ray
passing through the cell. Hits are recorded by rounding point clouds into voxel
coordinates, subject to the exclusions above. Maps update above 60 Hz, exceeding
both the ZED2i publish rate and the LiDAR's 20 Hz bound.

\subsubsection{Sample Injection} \label{sec:perceptionsmaps} To deploy the CBF-QP with the guarantees of Proposition \ref{cor:injection}, we require a bounding volume of the grasped object $\hat{\mathcal{B}}_O^E$  even under perturbation and sensing uncertainty. We define $d$ to be the maximum possible distance that any point on the held object's surface can move relative to the end-effector frame during the grasp. If, at any time during the grasp, all points $d$ away from the surface are in $\hat{\mathcal{B}}_O^E$, the held object will remain a subset for the grasp duration. This motivates our injected sample initialization, and carving. Given an object bounding box and some assumed $d$, inflating the box by $d$ yields a valid, conservative $\hat{\mathcal{B}}_O^E$. Points in the set can be removed at any time if there is no occupancy within $d$. In practice, we initialize with a $d$-expanded voxelized bounding box from YOLO \cite{ultralytics} and carve from perception data, resulting in $\hat{\mathcal{B}}_O^E$ \& object samples sets as shown in yellow and pink, respectively, in Fig \ref{fig:hero}. These samples can then be used to prevent marking the held object as occupied via the same sample method described above, preventing catastrophic self-evasion. The adaptability of this fast, sample based exclusion method is a direct benefit over using model based occupancy elimination. 

\subsubsection{VLA Training and Deployment} \label{vlahard}

To demonstrate our filter's operation over a learned nominal controller on hardware, we fine-tune $\pi_{0.5}$~\cite{intelligence2025pi05visionlanguageactionmodelopenworld} from the released base checkpoint on $430$ teleoperated demonstrations of a tabletop pick-and-place task (placing a banana into a bowl) collected on a Franka FR3 with compliant Fin Ray fingers. The policy was not trained with obstacles and no safety objective enters at training. The policy is queried at $10$Hz for an action chunk of length $16$ which is the nominal input. Occupancy is constructed from the perception stack of Sec.~\ref{sec:perception} at $40$ Hz. 
The coarse $80^3$ grid at a $3$cm resolution covers the entire workspace, and the fine $50^3$ grid at a $1$cm resolution covers the task subset. 
Poisson's equation is solved on both of these grids to produce the two safety functions of the space at $40$Hz. The safety filter runs at $150$Hz with the most recent Poisson solutions. Pick up and placement target bounding boxes, from YOLO \cite{ultralytics}, are used to clear the occupancy map, allowing for robot interaction with the objects. Currently, targets are identified from the task description, but could be semantically selected by a VLM. 

\subsubsection{Demonstration in Clutter}
First, we show the effectiveness of our method for VLA navigation in clutter. The nominal policy is finetuned only on unobstructed paths. As seen in the upper panel of Figure \ref{fig:clutter-strip}A, the robot collides with multiple obstacles in the scene and obstructs the placement object with the nominal VLA. In contrast, our method can safely navigate through clutter with a naive VLA policy. In the bottom panel of Figure \ref{fig:clutter-strip}A, our safety architecture is deployed on top of the nominal policy and  safely and successfully complete the task. In clutter, precise representations of the environment are required for navigation, especially near the obstacles, where the fine grid at a high resolution becomes engaged (see Figure \ref{fig:finevscoarse}).

\subsubsection{Dynamic Avoidance}

The method extends to dynamic scenes when using a nominal VLA policy. We extend the 3D perception stack from \cite{wilkinson2026fullbodydynamicsafetyrobot} for arbitrary dynamic environments. There is a sensing and reachability bound for the obstacle speed the filter can accommodate. First, the occupancy grid is updated at a finite rate, so an obstacle moving faster than the occupancy-to-PSF pipeline is represented with a delay proportional to its velocity. Even under exact occupancy, a fast moving obstacle can induce actions that require exceeding the robot's velocity limits which cannot be satisfied. Figure \ref{fig:clutter-strip}B shows the PSF value across a
continuous trial in which the FR3 completes two successful pick-and-place tasks while a human and ball obstacle repeatedly enter the workspace and attempt to interfere with the task. Grasped object sample points extend safety from the full robot body to the held banana throughout. The PSF value maintains its positive value throughout and task progress resumes once the workspace clears.


\section{Conclusion}

We have presented an architecture for full-body dynamic safety of VLA policies without retraining. The safety filter is modular and can be applied to general learned policies. Across a subset of tasks, the safety filter improves task success by preventing episode-ending collisions, while other tasks are hindered. Our method improved full-body collision avoidance in all cases.  We extend earlier results to provide safety guarantees for grasped objects and develop the perception and computational pipeline to enable real-time validated operation in dynamic and cluttered environments.

Future work will investigate how to close the success gap in the failure modes, perhaps by utilizing the VLA's plan in combination with modern control methods such as Model Predictive Control, or integrating safety earlier into the learned pipeline. The current perception pipeline is limited by observation occlusions, better shape prediction could improve or maintain safety under lower sensor burdens.



\vspace{2mm}
\noindent\textbf{Acknowledgments} - This research was supported by The Dow Chemical Company through project \#227027AW. Claude Opus 5 aided in coding and analysis.

\addtolength{\textheight}{-12cm}   





\bibliographystyle{IEEEtran}

\bibliography{main-GB}

\end{document}